\documentclass{article}
\usepackage{spconf,amsmath,amsthm,graphicx}

\usepackage{shorthands}
\usepackage{xcolor}
\usepackage{fancyhdr}
\usepackage{subcaption}
\usepackage{mathtools}
\usepackage{algorithm}
\usepackage{algpseudocode}

\title{Score-Based Methods for Discrete Optimization in Deep Learning}
\name{Eric Lei$^1$ \qquad Arman Adibi$^2$ \qquad Hamed Hassani$^1$  }
\address{$^1$Department of Electrical and Systems Engineering, University of Pennsylvania \\
$^2$Department of Electrical Engineering, Princeton  University}

\begin{document}
\maketitle
\begin{abstract}
Discrete optimization problems often arise in deep learning tasks, despite the fact that neural networks typically operate on continuous data. One class of these problems involve objective functions which depend on neural networks, but optimization variables which are discrete. Although the discrete optimization literature provides efficient algorithms, they are still impractical in these settings due to the high cost of an objective function evaluation, which involves a neural network forward-pass. In particular, they require $O(n)$ complexity per iteration, but real data such as point clouds have values of $n$ in thousands or more. In this paper, we investigate a score-based approximation framework to solve such problems. This framework uses a score function as a proxy for the marginal gain of the objective, leveraging embeddings of the discrete variables and speed of auto-differentiation frameworks to compute backward-passes in parallel. We experimentally demonstrate, in adversarial set classification tasks, that our method achieves a superior trade-off in terms of speed and solution quality compared to heuristic methods. 
\end{abstract}
\section{Introduction}
\label{sec:intro}

Although deep learning (DL) primarily operates on continuous data, \textit{discrete} optimization problems often arise in many DL applications, especially in data with discrete aspects (e.g. sets, graphs, and text). In these problems, the data that is operated on contains structure in the form of a set or a sequence $S$. Some examples may include graph nodes or point cloud points for sets, and tokens in a text sequence for sequences. For these examples, relevant tasks may include saliency computation of different elements in $S$, or to find adversarial sets/sequences $S$ corresponding to some predictor of $S$. 
    
Many of these problems can be formulated as a subset selection problem, which has the form 
\begin{equation}
    \max_{S' \subseteq S} \varphi(S'),
    \label{eq:general}
\end{equation}
where $\varphi$ is an objective function of the set or sequence. In DL applications, $\varphi$ typically depends on a neural network. For example, in finding adversarial subsets or subgraphs in set and graph classification, $\varphi$ can be viewed as the error function of a classifier $f_{\theta}(S)$, i.e., $\varphi(S) = \ell(y, f_\theta(S))$.

\begin{figure}[t]
     \centering
     \begin{subfigure}[b]{0.49\linewidth}
         \centering
         \includegraphics[width=0.7\linewidth]{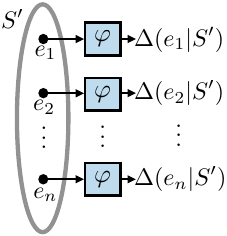}
         \caption{Exact marginal gain.}
         \label{fig:marginal_gain_computation}
     \end{subfigure}
     \hfill
     \begin{subfigure}[b]{0.49\linewidth}
         \centering
         \includegraphics[width=0.75\linewidth]{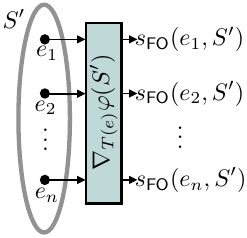}
         \caption{Score-based approximation.}
         \label{fig:score_computation}
     \end{subfigure}
        \caption{Computing exact marginal gain $\Delta(e|S') \coloneqq \varphi(S' \setminus e) - \varphi(S')$ requires $O(n)$ forward-passes. The score-based surrogate requires a single forward and backwards pass.}
        \label{fig:intro_fig}
\end{figure}

This aspect of the objective function makes solving \eqref{eq:general} inefficient, since evaluating $\varphi$ may require large computational resources.  Exactly solving \eqref{eq:general} is NP-hard in general. Many approximation algorithms exist that sacrifice exact solutions in exchange for speed; however, these are still too slow in practice.  While approximation methods such as greedy iterations yield approximation guarantees for \eqref{eq:general}, they require $O(n)$ function evaluations of $\varphi$ per iteration, where $n = |S|$. Since evaluating $\varphi$ requires a forward-pass through a neural network, these approximation methods hence require $O(n)$ neural network evaluations per iteration. This is clearly intractable in practical settings, where the number of elements in a point cloud or nodes in a graph are routinely at least on the order of $10^{4}$ \cite{chang2015shapenet, wu20153d, hu2020open}.

Rather than treat $\varphi$ as a black-box function and apply standard approximation methods, we propose to leverage the structure of $\varphi$ to design more efficient algorithms. Since $\varphi$ highly depends on DL architectures, we can use the embedding space of the functions parameterized by the DL models to inform the optimization solver. We propose a principled method that approximates the marginal gain of $\varphi$ using a first-order approximation in the embedding space, where removing an element can be simulated using \textit{uninformative embeddings}. This method generalizes several ad-hoc methods, and leverages the auto-differentiation and parallelization of $\varphi$ to compute gradients. 

In what follows, we discuss past works and related approaches in Sec.~\ref{sec:related_work} before introducing our score-based discrete optimization framework in Sec.~\ref{sec:framework}. Then, we apply our method to real-world settings, such as adversarial subset selection in set prediction, and saliency score computation, in Sec.~\ref{sec:results}.

\section{Related Work}
\label{sec:related_work}
\begin{algorithm}[t]
    \caption{Iterative Subset Selection}\label{alg:iterative}
        \begin{algorithmic}
            \Require Score function $s : S \times 2^S \rightarrow \mathbb{R}$, steps $k$
            \State Initialize $S' = S$
            \For{$i=1,\dots,k$}
                \State $e^* = \argmax_{e \in S'} s(e, S')$ 
                \State $S' \leftarrow S' \setminus e^*$ 
            \EndFor
            \State Return $S'$
        \end{algorithmic}
\end{algorithm}

\textbf{Deep Learning On Sets}: Deep learning architectures that operate on sets were first prominently proposed by the authors of DeepSets \cite{zaheer2017deep} and PointNet \cite{qi2017pointnet}, as well as follow-up works \cite{qi2017pointnet++}. These architectures consume a set $S$, and can be easily adapted to (i) set classification or (ii) element classification (e.g., segmentation) settings. In both cases, they first a shared feature space for each element of the set using a permutation-equivariant architecture such as a point-based MLP \cite{zaheer2017deep, qi2017pointnet}, or graph neural network \cite{qi2017pointnet++}. For (i), a global pooling layer is then used, which ensures that a shared feature space is learned across sets of varying sizes. This also enforces permutation-invariance within the model, which is ideal since sets do not depend on the ordering of the elements. For (ii), no global pooling is used, and instead the elements are classified directly from the element feature space. 

\textbf{Adversarial Subset Selection}:
Selecting adversarial subsets of a set classifier $f_{\theta}$ is a natural application of \eqref{eq:general} such that the objective $\varphi$ depends on a neural network $f_{\theta}$. Though not initially motivated by this application, the authors of \cite{zheng2019pointcloud} first proposed methods to find adversarial subsets of point cloud classifiers. This method defines a heuristic saliency score for each point $i$, where the saliency score simulates the effect of deleting $i$ from the point cloud. The motivation is behind this is that the global pooling of the classifiers will ``focus'' on a particular subset of points in order to discern shape. Since the points of an object point cloud lie on a 2-manifold, shifting point $i$ to an uninformative region (such as the origin) will cause the global pooling to ignore point $i$, and hence simulate a deletion. Other works that discuss adversarial point cloud attacks specifically \cite{zhou2019dup, yang2019adversarial} use variants of this as a deletion attack to design defenses. In this work, we demonstrate that these methods are actually performing an approximation of the greedy algorithm in set function maximization. Then, we show how the saliency score can be generalized to more general score functions that approximate the marginal gain.

\textbf{Other Applications Involving Subset Selection}:
Other applications that fit the framework in \eqref{eq:general} can include coreset selection \cite{coleman2019selection, rangwani2021s3vaada, dolatabadi2023adversarial} where one must select a subset of a dataset for some task. These works also use some objective function $\varphi$ which typically depends on a neural network for the downstream task of interest, and apply a variant of iterative greedy algorithms. 

\section{Score-Based Greedy Approximation}
\label{sec:framework}

We first discuss greedy methods for solving the subset maximization problem in \eqref{eq:general}, before introducing the score-based framework. To solve the problem in \eqref{eq:general}, the popular approach involves an iterative procedure which sequentially removes elements from the set. This picks $k$ elements to remove in a sequential manner, such that at each step, the element picked is the one that maximizes a score function $s : S \times 2^S \rightarrow \mathbb{R}$. 
Hence, the update at each step performs 
\begin{equation}
    S' \leftarrow S' \setminus \{\argmax_{e \in S'} s(e, S')\},
\end{equation}
where $S'$ is initialized to be $S$. 

\subsection{Greedy Methods}
Traditionally, the function used for the score $s$ is the marginal gain of removing element $e$ from $S'$, defined to be
\begin{equation}
    \Delta(e|S') \coloneqq \varphi(S' \setminus e) - \varphi(S').
\end{equation}
Using the marginal gain, i.e. $s(e, S') = \Delta(e|S')$, in Alg.~\ref{alg:iterative}, is known as the greedy algorithm. This algorithm provides a $1-1/e$ approximation of the optimal solution of \eqref{eq:general} when $\varphi$ is a monotone submodular set function and there is a cardinality constraint of $|S'| \geq |S|-k$ in \eqref{eq:general}.
Even when $\varphi$ is not submodular, the greedy algorithm can provide good performance \cite{bian2017guarantees}.

However, as mentioned in Sec.~\ref{sec:intro}, finding $\argmax_e \Delta(e|S')$ requires $O(n)$ function evaluations, where $n = |S'|$. In many non-deep learning applications, this is acceptable; however, when $\varphi$ depends on a neural network forward pass to evaluate, this can be impractical for point clouds or sets containing thousands to millions of points. It is also not easy to parallelize this operation, since one would need to create a superset of $n$ total sets $S' \setminus e, e \in S'$, to input to $\varphi$. This would require $O(n^2)$ elements to be stored in memory, which is infeasible for most commodity GPUs. Additionally, various algorithms applied in discrete optimization, such as replacement greedy, gradient descent on the multilinear extension of discrete functions, and conditional gradient on the multilinear extension of discrete functions, require $O(n)$ function evaluations \cite{adibi2022minimax, hassani2017gradient, zhang2020one}.

\subsection{Score-based Approximation}
Rather than use the marginal gain function, we argue that replacing $\Delta(e|S')$ with a general score function $s$ that acts as a surrogate for $\Delta$ can still perform well, yet with much better computational speed. Our goal is to find some function $s$ such that 
\begin{equation}
    s(e, S') \approx \Delta(e|S') = \varphi(S' \setminus e) - \varphi(S').
\end{equation}
In general, it can be unclear how to find such a function. However, in most problems of interest, the $\varphi$ function (which depends on a neural network), typically depends on the set elements $e \in S'$ via associated features or embeddings, which we will denote $T(e)$. While $e$ is typically just an index denoting the element in the set, $T(e)$ carries the information contained in the set depending on the application of interest. For example, in point clouds, each element $e \in S'$ is represented by a 3-dimensional coordinate which contains the geometry information of the object. Similarly, graphs nodes typically carry node features, and text data is usually tokenized and converted in a sequence of embedding vectors. The features $T(e)$ are Euclidean vectors which we can compute gradients with respect to. 

Since $\varphi$ processes the set $S' = \{e_1,\dots,e_n\}$ via its embedding set $\mathcal{T}_{S'} = \{T(e_1), \dots, T(e_n)\}$, we define the auxiliary function $\bar{\varphi}$ that directly depends on the embddings,
\begin{equation}
    \bar{\varphi}(\mathcal{T}_{S'}) \coloneqq \varphi(S).
\end{equation}
Thus, $\bar{\varphi}$ and $\varphi$ can be used interchangeably depending on the arguments one wishes to use. 
One can use a first-order Taylor expansion of $\bar{\varphi}$, 
\begin{equation}
    \bar{\varphi}(\mathcal{T}_{\tilde{S}}) \approx \bar{\varphi}(\mathcal{T}_{S'}) + \nabla_{\mathcal{T}_{S'}} \bar{\varphi}(\mathcal{T}_{S'})^\top (\mathcal{T}_{\tilde{S}} - \mathcal{T}_{S'}),
\end{equation}
where $\mathcal{T}_{\tilde{S}} = \{T_1, \dots, T_n\}$ is some set of $n$ embeddings corresponding to a set $\tilde{S}$. Here, we overload notation for the term $\nabla_{\mathcal{T}_{S'}} \bar{\varphi}(\mathcal{T}_{S'})^\top (\mathcal{T}_{\tilde{S}} - \mathcal{T}_{S'})$ by treating $\mathcal{T}_{\tilde{S}}$ and $\mathcal{T}_{S'}$ as ordered vectors, even though they are sets. This is valid since both sets are of the same size, and the inner product is a summation operation. 
Rearranging, we have that 
\begin{equation}
    \bar{\varphi}(\mathcal{T}_{\tilde{S}}) - \bar{\varphi}(\mathcal{T}_{S'}) \approx -\nabla_{\mathcal{T}_{S'}} \bar{\varphi}(\mathcal{T}_{S'})^\top (\mathcal{T}_{S'} - \mathcal{T}_{\tilde{S}}).
    \label{eq:score1}
\end{equation}
Note that the left hand side is equivalent to $\varphi(\tilde{S}) - \varphi(S')$. For this to fit the marginal gain, the set $\tilde{S}$ would need to approximate the set $S' \setminus e$. However, computing the inner product in \eqref{eq:score1} requires that $\tilde{S}$ and $S'$ have the same number of elements, which is clearly not possible if $\tilde{S} = S' \setminus e$. One way to avoid this issue is to choose $\tilde{S}$ (and its embedding set $\mathcal{T}_{\tilde{S}}$) carefully in order to simulate a deletion of element $e$ from $S'$. 

\subsection{Uninformative Embeddings}
As explained in the previous section, in the settings we are interested in, the function $\varphi$ depends directly on the embedding set $\mathcal{T}_{S'}$. In order to simulate a deletion of some element $e$, one can choose to perturb the embedding $T(e)$ in $\mathcal{T}_{S'}$ so that it gets ``ignored'' by the computation of $\varphi$. If this is the case, such a perturbation of $T(e)$, call it $T'(e)$, effectively simulates inputting $\mathcal{T}_{S'} \setminus T(e)$ to $\bar{\varphi}$. Since $T'(e)$ gets ignored from $\varphi$'s computation, we refer to $T'(e)$ as an \textit{uninformative embedding}.

The question now is how to obtain an uninformative embedding $T'(e)$. To this end, the following proposition describes a class of functions $\varphi$ for which an uninformative embedding always exists for each element of the set.
\begin{proposition}
    Suppose that $\bar{\varphi}$ has the form
    \begin{equation}
        \bar{\varphi} (\mathcal{T}_{S'}) = g_2 \left (\max_{e \in S'} g_1(T(e)) \right),
        \label{eq:form}
    \end{equation}
    where $g_1: \mathbb{R}^d \rightarrow \mathbb{R}^k$ and $g_2 : \mathbb{R} \rightarrow \mathbb{R}$ is some scalar function.
    If $T'(e)$ satisfies $g_1(T'(e)) < g_1(T(\hat{e})), \forall \hat{e} \in S'$,
    \begin{equation}
        \bar{\varphi}(\{T(e_1),\dots,T'(e),\dots,T(e_n)\}) = \bar{\varphi}(\mathcal{T}_{S'} \setminus T(e)),
    \end{equation}
    where the LHS input is $\mathcal{T}_{S'}$ with $T(e)$ replaced with $T'(e)$.
\end{proposition}
\begin{proof}
    By the assumption, the vector $g_1(T'(e))$ will not be selected by the feature-wise maximum in \eqref{eq:form}. Hence with $\{T(e_1),\dots,T'(e),\dots,T(e_n)\}$ as input, $\bar{\varphi}$ will ignore $T'(e)$ and effectively compute $\bar{\varphi}(\mathcal{T}_{S'} \setminus T(e))$.
\end{proof}
Many permutation-invariant set deep learning models, such as PointNet and DeepSets with max-pooling layers, induce functions $\varphi$ that satisfy Prop.~1. 

Hence, in order for \eqref{eq:score1} to provide a first-order approximation of the marginal gain of removing element $e$ (i.e., $\Delta(e|S')$), it suffices to choose 
\begin{equation}
    \mathcal{T}_{\tilde{S}} = \mathcal{T}_{\tilde{S}(e)} \coloneqq \{T(e_1),\dots,T'(e),\dots,T(e_n)\}.
\end{equation}
Hence, our proposed score $s_{\mathsf{FO}}$ to first-order approximate $\Delta(e|S')$ is defined as
\begin{align}
    s_{\mathsf{FO}}(e, S') &\coloneqq -\nabla_{\mathcal{T}_{S'}} \bar{\varphi}(\mathcal{T}_{S'})^\top (\mathcal{T}_{S'} - \mathcal{T}_{\tilde{S}})  \nonumber \\
    &= -\nabla_{T(e)}\bar{\varphi}(\mathcal{T}_{S'})^\top (T(e) - T'(e)). \label{eq:score}
\end{align}

\begin{remark}
    $s_{\mathsf{FO}}(e, S')$ admits computationally-friendly implementations compared to $\Delta(e|S')$. To compute $\max_{e \in S'} s_{\mathsf{FO}}(e, S')$, one only needs to compute a single forward pass through $\varphi$ with $S'$ as input, and one backwards pass to get all of the gradients $-\nabla_{T(e)}\bar{\varphi}(\mathcal{T}_{S'})$ in \emph{parallel}. Since $S'$ is the only input set, this approach only uses $O(n)$ memory. 
\end{remark}

\begin{remark}
    The saliency proposed in \cite{zheng2019pointcloud}, can be written as
    \begin{equation}
        \bs_i = - \|\bx_i - \bar{X}\|_2^2 \nabla_{\bx_i} \ell(y, f_{\theta}(X))(\bx_i - \bar{X}),
    \end{equation}
    which is the saliency score for point $\bx_i$ in a point cloud $X = \{\bx_1,\dots,\bx_n\}$, where $\bar{X}$ is the median of $X$. In their work, $\bs_i$ was derived heuristically using the fact that point clouds have points $\bx$ that lie on a surface, and the classifier $f_{\theta}$ will tend to classify based on points on the boundary, and not near the center (i.e., median) of the point cloud. The median, in this case, corresponds to the uninformative embedding that we have proposed. Hence, our approach using $s_{\mathsf{FO}}$ generalizes the saliency score, by showing that it approximates a first-order approximation of the marginal gain. 
\end{remark}

To find subsets, we may apply Alg.~\ref{alg:iterative}, but replace $\Delta(e|S')$ with $s_{\mathsf{FO}}(e, S')$ as the score function. 

\section{Experimental Results}
\label{sec:results}

\subsection{Experimental Setup}
In order to evaluate $s_{\mathsf{FO}}$, we examine adversarial subset selection for set classification. In particular, we look at point cloud classifiers $f_{\theta}$, which predicts classes $\hat{y} = f_{\theta}(S)$. In all experiments, we use PointNet \cite{qi2017pointnet} as the classifier $f_{\theta}$. We use the ModelNet40 \cite{wu20153d} dataset. Since the shapes in ModelNet40 are provided as triangle meshes, we sample 1024 points uniformly from the surface defined by the mesh. Hence, the sets $S$ each represent a point cloud of $n=1024$ points, with embeddings $\mathcal{T}_S$ as the corresponding set of 3D coordinates. 

To find adversarial subsets, we solve
\begin{equation}
    S^*_k = \argmax_{S' \in 2^S: |S'| \geq n-k} \ell(y, f_{\theta}(S')) 
\end{equation}
where $\ell$ is the cross-entropy loss. We compare greedy and score-based methods, as well as random (remove a random sample in each iteration), and hybrid methods (the score is used to create a $m$-sized candidate set, of which the max marginal gain is chosen from). For the score-based methods, we use the median as the uninformative embeddings as in \cite{zheng2019pointcloud}, as well as a feature-space embedding in PointNet:
$s(e, S') = -  \nabla_{\phi(\bx_i)}\ell(y,f_{\theta}(S'))^\top (\phi(\bx_i) - \phi_{\text{min}})$, where $\phi(\cdot)$ is the pointwise features of PointNet (i.e., the output of the pointwise MLP), and $\phi_{\text{min}}$ is the pointwise minimum.

\subsection{Results}
\begin{figure}[t]
    \centering
    \includegraphics[width=0.75\linewidth]{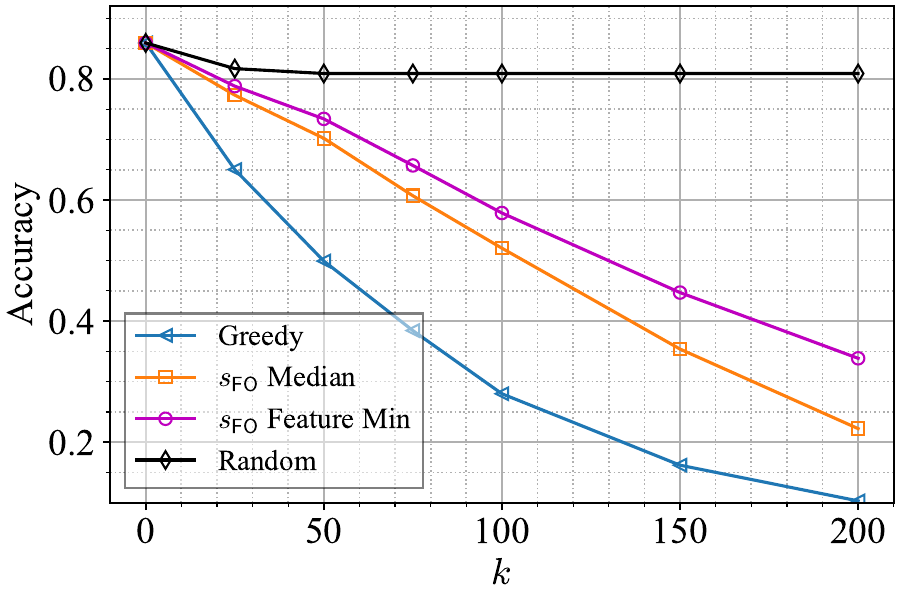}
    \vspace{-1em}
    \caption{\small Adversarial subset selection in point cloud classification.}
    \label{fig:greedy_comparison}
\end{figure}
\begin{figure}[t]
    \centering
    \includegraphics[width=0.75\linewidth]{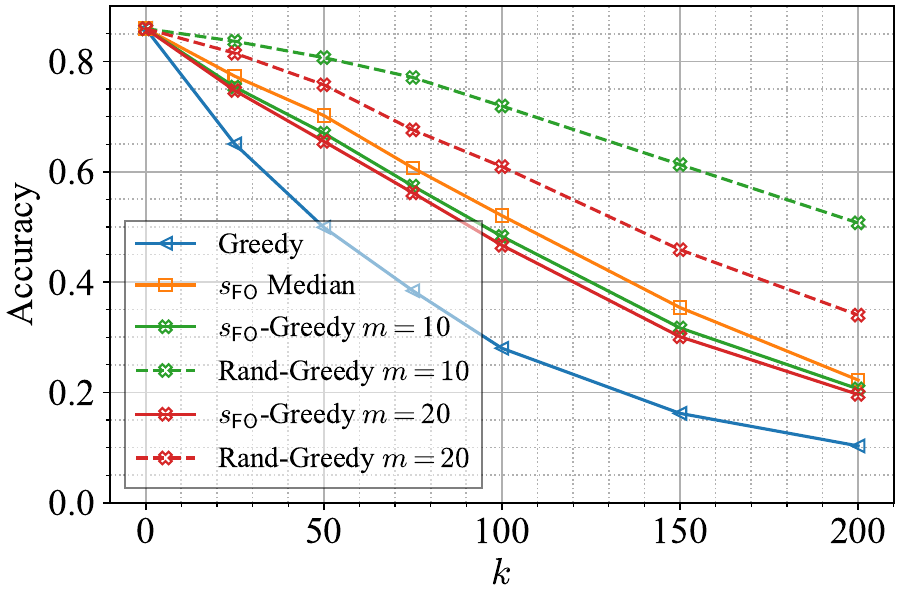}
    \vspace{-1em}
    \caption{\small Score-guided greedy methods. }
    \label{fig:hybrid_comparison}
\end{figure}

\begin{figure}[t]
    \centering
    \includegraphics[width=0.75\linewidth]{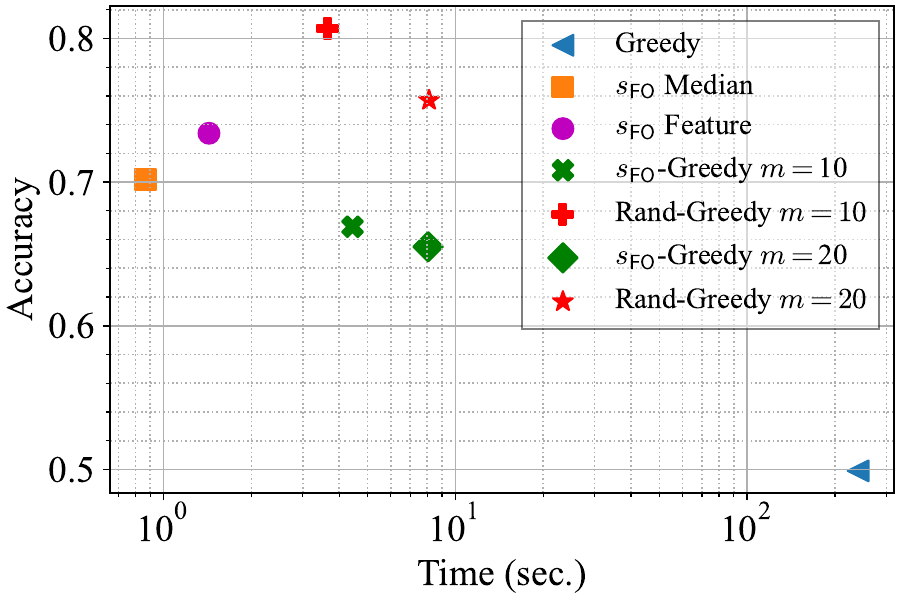}
    \vspace{-1em}
    \caption{\small Performance-speed tradeoffs. Lower is better for both axes. }
    \label{fig:complexity}
\end{figure}

We first show the drop in classification accuracy (via the adversarial subset) as a function of $k$ using the different methods in Fig.~\ref{fig:greedy_comparison}. We see clearly that the greedy method performs best in finding subsets that reduce the accuracy. Randomly deleting points fails to decrease the accuracy. The $s_{\mathsf{FO}}$ methods are able to find adversarial subsets, with the median method performing around 5\% better. This is likely due to the fact that PointNet has a spatial transformer layer prior to the pointwise MLP, which makes it not perfectly fit the form in Prop.~1. 

For the hybrid methods (Fig.~\ref{fig:hybrid_comparison}), where the score (or random) is used to create a candidate set of size $m$ for the greedy method, we see that this helps bridge the gap between greedy and $s_{\mathsf{FO}}$ methods. However, there are diminishing returns with increasing $m$. Using random candidates decreases the accuracy, but it does not do as well as just using $s_{\mathsf{FO}}$.

Finally, we demonstrate performance-speed tradeoffs in Fig.~\ref{fig:complexity}. We plot accuracy of the selected subset vs. average time taken to compute the adversarial subset; we set $k=50$, batch size of 32, and test on a RTX 5000 GPU. Greedy performs the best, but requires several orders of magnitude more time. For the other methods, the $s_{\mathsf{FO}}$-based methods yield superior tradeoffs compared to random-candidate greedy.

\section{Conclusion}
\label{sec:conclusion}
This paper proposes a score-based approximation of marginal gain for solving discrete optimization problems in deep learning scenarios. The proposed method generalizes prior heuristic methods, and allows one to easily traverse superior performance-speed tradeoffs compared to other methods. 

\bibliographystyle{IEEEbib}
\bibliography{strings,refs}

\begin{thebibliography}{10}

\bibitem{chang2015shapenet}
Angel~X Chang, Thomas Funkhouser, Leonidas Guibas, Pat Hanrahan, Qixing Huang,
  Zimo Li, Silvio Savarese, Manolis Savva, Shuran Song, Hao Su, et~al.,
\newblock ``Shapenet: An information-rich 3d model repository,''
\newblock {\em arXiv preprint arXiv:1512.03012}, 2015.

\bibitem{wu20153d}
Zhirong Wu, Shuran Song, Aditya Khosla, Fisher Yu, Linguang Zhang, Xiaoou Tang,
  and Jianxiong Xiao,
\newblock ``3d shapenets: A deep representation for volumetric shapes,''
\newblock in {\em Proceedings of the IEEE conference on computer vision and
  pattern recognition}, 2015, pp. 1912--1920.

\bibitem{hu2020open}
Weihua Hu, Matthias Fey, Marinka Zitnik, Yuxiao Dong, Hongyu Ren, Bowen Liu,
  Michele Catasta, and Jure Leskovec,
\newblock ``Open graph benchmark: Datasets for machine learning on graphs,''
\newblock {\em Advances in neural information processing systems}, vol. 33, pp.
  22118--22133, 2020.

\bibitem{zaheer2017deep}
Manzil Zaheer, Satwik Kottur, Siamak Ravanbakhsh, Barnabas Poczos, Russ~R
  Salakhutdinov, and Alexander~J Smola,
\newblock ``Deep sets,''
\newblock {\em Advances in neural information processing systems}, vol. 30,
  2017.

\bibitem{qi2017pointnet}
Charles~R Qi, Hao Su, Kaichun Mo, and Leonidas~J Guibas,
\newblock ``Pointnet: Deep learning on point sets for 3d classification and
  segmentation,''
\newblock in {\em Proceedings of the IEEE conference on computer vision and
  pattern recognition}, 2017, pp. 652--660.

\bibitem{qi2017pointnet++}
Charles~Ruizhongtai Qi, Li~Yi, Hao Su, and Leonidas~J Guibas,
\newblock ``Pointnet++: Deep hierarchical feature learning on point sets in a
  metric space,''
\newblock {\em Advances in neural information processing systems}, vol. 30,
  2017.

\bibitem{zheng2019pointcloud}
Tianhang Zheng, Changyou Chen, Junsong Yuan, Bo~Li, and Kui Ren,
\newblock ``Pointcloud saliency maps,''
\newblock in {\em Proceedings of the IEEE/CVF International Conference on
  Computer Vision}, 2019, pp. 1598--1606.

\bibitem{zhou2019dup}
Hang Zhou, Kejiang Chen, Weiming Zhang, Han Fang, Wenbo Zhou, and Nenghai Yu,
\newblock ``Dup-net: Denoiser and upsampler network for 3d adversarial point
  clouds defense,''
\newblock in {\em Proceedings of the IEEE/CVF International Conference on
  Computer Vision}, 2019, pp. 1961--1970.

\bibitem{yang2019adversarial}
Jiancheng Yang, Qiang Zhang, Rongyao Fang, Bingbing Ni, Jinxian Liu, and
  Qi~Tian,
\newblock ``Adversarial attack and defense on point sets,''
\newblock {\em arXiv preprint arXiv:1902.10899}, 2019.

\bibitem{coleman2019selection}
Cody Coleman, Christopher Yeh, Stephen Mussmann, Baharan Mirzasoleiman, Peter
  Bailis, Percy Liang, Jure Leskovec, and Matei Zaharia,
\newblock ``Selection via proxy: Efficient data selection for deep learning,''
\newblock {\em arXiv preprint arXiv:1906.11829}, 2019.

\bibitem{rangwani2021s3vaada}
Harsh Rangwani, Arihant Jain, Sumukh~K Aithal, and R~Venkatesh Babu,
\newblock ``S3vaada: Submodular subset selection for virtual adversarial active
  domain adaptation,''
\newblock in {\em Proceedings of the IEEE/CVF International Conference on
  Computer Vision}, 2021, pp. 7516--7525.

\bibitem{dolatabadi2023adversarial}
Hadi~M Dolatabadi, Sarah~M Erfani, and Christopher Leckie,
\newblock ``Adversarial coreset selection for efficient robust training,''
\newblock {\em International Journal of Computer Vision}, pp. 1--25, 2023.

\bibitem{bian2017guarantees}
Andrew~An Bian, Joachim~M Buhmann, Andreas Krause, and Sebastian Tschiatschek,
\newblock ``Guarantees for greedy maximization of non-submodular functions with
  applications,''
\newblock in {\em International conference on machine learning}. PMLR, 2017,
  pp. 498--507.

\bibitem{adibi2022minimax}
Arman Adibi, Aryan Mokhtari, and Hamed Hassani,
\newblock ``Minimax optimization: The case of convex-submodular,''
\newblock in {\em International Conference on Artificial Intelligence and
  Statistics}. PMLR, 2022, pp. 3556--3580.

\bibitem{hassani2017gradient}
Hamed Hassani, Mahdi Soltanolkotabi, and Amin Karbasi,
\newblock ``Gradient methods for submodular maximization,''
\newblock {\em Advances in Neural Information Processing Systems}, vol. 30,
  2017.

\bibitem{zhang2020one}
Mingrui Zhang, Zebang Shen, Aryan Mokhtari, Hamed Hassani, and Amin Karbasi,
\newblock ``One sample stochastic frank-wolfe,''
\newblock in {\em International Conference on Artificial Intelligence and
  Statistics}. PMLR, 2020, pp. 4012--4023.

\end{thebibliography}

\end{document}